\documentclass[a4paper]{article}
\pdfoutput=1
\usepackage{dblfloatfix}
\usepackage{caption}  
\usepackage{graphicx}
\usepackage[numbers]{natbib}
\usepackage{tikz}
\usetikzlibrary{shapes,arrows,calc,positioning}
\usepackage{amssymb}
\usepackage{amsmath}
\usepackage{mathtools}
\usepackage{algorithm}
\usepackage{algorithmicx}
\usepackage{algpseudocode}
\usepackage{subfigure}
\usepackage{paralist}
\usepackage{listings}
\usepackage{xcolor}
\usepackage[export]{adjustbox}
\usepackage{url}
\usepackage{csquotes}
\usepackage{color, colortbl}
\usepackage{mathdots}
\usepackage{tabularx}

\usepackage{amsthm}
\newtheorem{theorem}{Theorem}[subsection]
\newtheorem{lemma}[theorem]{Lemma}

\newtheorem{definition}{Definition}[subsection]

\newtheorem{remark}{Remark}[subsection]
\newtheorem{example}{Example}[subsection]

\definecolor{Gray}{gray}{0.9}

\begin{document}
\title{Theoretical Robopsychology: Samu Has Learned Turing Machines}
\author{Norbert~B\'atfai\\batfai.norbert@inf.unideb.hu\\Department of Information Technology\\University of Debrecen}

\maketitle

\begin{abstract}
From the point of view of a programmer, the robopsychology is a synonym for the activity is done by developers to implement their machine learning applications. This robopsychological approach raises some fundamental theoretical questions of machine learning. Our discussion of these questions is constrained to Turing machines. Alan Turing had given an algorithm (aka the Turing Machine) to describe algorithms. If it has been applied to describe itself then this brings us to Turing's notion of the universal machine. In the present paper, we investigate algorithms to write algorithms. 
From a pedagogy point of view, this way of writing programs can be considered as a combination of learning by listening and learning by doing due to it is based on applying agent technology and machine learning. As the main result we introduce the problem of learning and then we show that it cannot easily be handled in reality therefore it is reasonable to use machine learning algorithm for learning Turing machines.
\end{abstract}

\section{Introduction}

Samu is a disembodied developmental robotic experiment to develop a family chatterbot agent who will be able to talk in a natural language like humans do \cite{samu}. At this moment it is only an utopian idea of the project Samu. The practical purpose of Samu projects is to develop computational mental organs that can support software agents to acquire higher-order knowledge from their input \cite{psamu1}. The activities have been conducted during the development of such mental organs may be considered as first efforts to create on demand the Asimovian profession called robopsychology\footnote{\url{https://en.wikipedia.org/wiki/Robopsychology}} \cite{robopsyc}.

The roots of this paper lie in the two new software experiments Samu Turing \cite{SamuTuring} and Samu C. Turing \cite{SamuCTuring}. These are very simplified versions of the former habituation-sensitization \cite{DevRobBook} based 
(like for example SamuBrain \cite{SamuBrain} or SamuKnows \cite{SamuKnows}) learning projects of Samu. 
Their common feature is that they use the same COP-based Q-learning engine that the chatbot Samu does.
To be more precise the mental organs use the same code 
(to see this compare
\url{https://github.com/nbatfai/SamuLife/blob/master/SamuQl.h} 
with \url{https://github.com/nbatfai/nahshon/blob/master/ql.hpp}) 
as the chatbot does. The term \enquote{COP-based} (Consciousness Oriented Programming \cite{cop}) means that the engine predicts its future input.
The engine itself is based on the Q-learning that receives positive reinforcement if the chatbot (or a mental organ) can successfully forecast the next input of Q-learning in the actual step. In this case the previous output (the previous prediction) is the same as the actual input, for precise details see  \cite{samu} and  \cite{psamu1}). 
In the two new experiments in question the transition rules of Turing machines (TMs) have been learned as it is illustrated in Fig \ref{fig_st}. 
It should be noticed that neither these experiments nor this paper focus on the habituation-based learning because the learning agent knows the model (TM) that generates the reality.  
\begin{figure}[!h]
\centering
\includegraphics[scale=.4, frame]{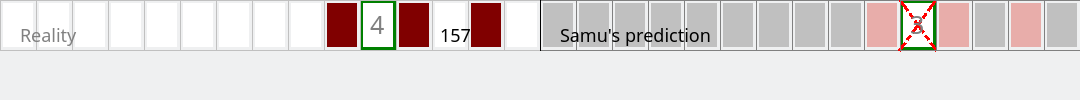}
\caption{This is a screenshot from the project Samu Turing. The reality shown in the left side is generated by the operation of a given Turing machine. The right side shows the predicted configurations of the investigated Turing machine.}
\label{fig_st}
\end{figure}
Our motivation to write this paper stems from the last paragraph of the work of Neumann on the general theory of automata \cite{neumann1} where Neumann had suggested that there is a complexity level above which the machines can reproduce themselves and even more complicated ones. 
Neumann investigated the self-reproducing automata \cite{neumann1} roughly a decade after Alan Turing had published his work on universal simulation theorem \cite{halting}.
The Turing machine is a precise form of the informal notion of the algorithm to describe algorithms. If this description algorithm has been applied to describe itself then this brings us to Turing's notion of the universal machine. 
In an intuitive sense we can say that Neumann replaced Turing's notion of simulation with the notion of reproduction. In this work we would like to replace the reproduction with the learning. To be more precise we investigate algorithms to write algorithms. For simplicity of our discussion the scope of this paper is constrained to Turing machines.
It should be noticed that we could have used other universal computing models such as the Cellular Automata. For example, the first mental organs had learned the Conway's Game of Life \cite{psamu1} (or see the YouTube video at \url{https://youtu.be/_W0Ep2HpJSQ}). But in spite of this, we chose Turing machines because they are closer to the programmers' intuition.

The structure of this paper is as follows: the next section introduces the basic notations. Then, in Sect. \ref{sect_learning} we present the results of two Samu-based developmental robotic software experiments to learn how Turing machines operate.  
Here we investigate some specific TMs. It should be noticed that some of them, such as the machines of Schult and Uhing or the Marxen and Buntrock's BB5 champion machine are famous in the field of the Rad{\'{o}} Tibor's Busy Beaver problem \cite{VitanyiBook}. It is worth noting that despite that this problem is a very interesting theoretical computer science problem we do not address it in this paper. We introduce of the learning problem and give the basic notions of this subject. Finally we present a new complexity measure called self-reproduction complexity and we show in Subsect. \ref{cc} that it is reasonable to use machine learning algorithm for learning Turing machines. The paper is closed by a short conclusion in which some possible directions for further work are pointed out.

\section{Notations and Technical Background}

Throughout both this article and our software experiments we use the definition of the Turing machine (TM) that was introduced in \cite{batfai-rt} and also used in \cite{batfai-orchmach} where the Turing machine was defined by a quadruple $T=(Q, 0, \{0,1\}, f)$ where 
$f:Q\times\{0,1\} \rightarrow Q\times\{0,1\}\times\{\leftarrow,\uparrow,\rightarrow\}$ is a partial transition function and $0 \in Q \subset \mathbb{N}$ is the starting state.  
As usual a configuration determines the actual state of the head, the position of the head and the contents of the tape.
With the notation of \cite{batfai-rt} a configuration can be written in the form $w_{before}[q>w_{after}$, where $w_{before}, w_{after} \in \{0,1\}^*$ and $q \in Q$.

In some proofs for simplicity's sake we use multitape Turing machines or the blank symbol \textvisiblespace \ on the tape (that is the tape alphabet is extended by the symbol \textvisiblespace). In addition, without limiting the generality, we may assume that halting Turing machines (with a given input) do not contain unused transition rules. The notation $T(x)<\infty$ denotes that the machine $T$ with the input $x$ halts.

\begin{definition}[configN]The word $b_N\dots b_1[q>a_0a_1\dots a_N$ over the alphabet $ \{0,1, [, >\} \cup Q$ where $a_i, b_j \in \{0,1\}$ is referred to as a configN configuration if there is a configuration  $w_{before}[q>w_{after}$ such that 
$w_{before}[q>w_{after}$ $=$
$w^,_{before}b_N\dots b_1[q>a_0a_1\dots a_Nw^,_{after}$.
\end{definition}

\begin{remark}[config$\infty$]In some cases, see for example Remark \ref{rem_configinfty}, we extend the definition of the configuration as follows $\text{\textvisiblespace}^\infty w_{before}[q>w_{after} \text{\textvisiblespace}^\infty$. In this sense a usual  configuration corresponds to a config$\infty$ configuration where $w^,_{before} = w^,_{after} = \lambda$ the empty word.
\end{remark} 

We may note that the release of the project Samu C. Turing used in Fig. \ref{fig_runninglearning} uses config4 configurations. 

\section{Learning by Listening and Doing\label{sect_learning}}

In the aforementioned projects Samu Turing and Samu C. Turing we programmed the Samu agent to work in a similar way as, for example, Professor James Harland did in his work \cite{Harland} where he observed and studied the configurations of Marxen and Buntrock's Busy Beaver champion machines \cite{MarxenBuntrock}. In our experiments the agent Samu observes (listening) the consecutive subconfigurations of a given investigated Turing machine and try to predict (doing) the next rule of the machine that will be applied. From this viewpoint this whole learning process can be seen as a way of learning by listening and doing where the listening part is the sensation of the agent and doing is the prediction of the agent. But the question may naturally be raised why should we use agent technology and machine learning algorithms to learn Turing machines? Our explicit answer is based on the following intuitive results and it will be found in Sect. \ref{cc}. 

\subsection{Some Intuitive Results\label{intuitive}}

\begin{figure}[!h]
\centering
\includegraphics[scale=.5]{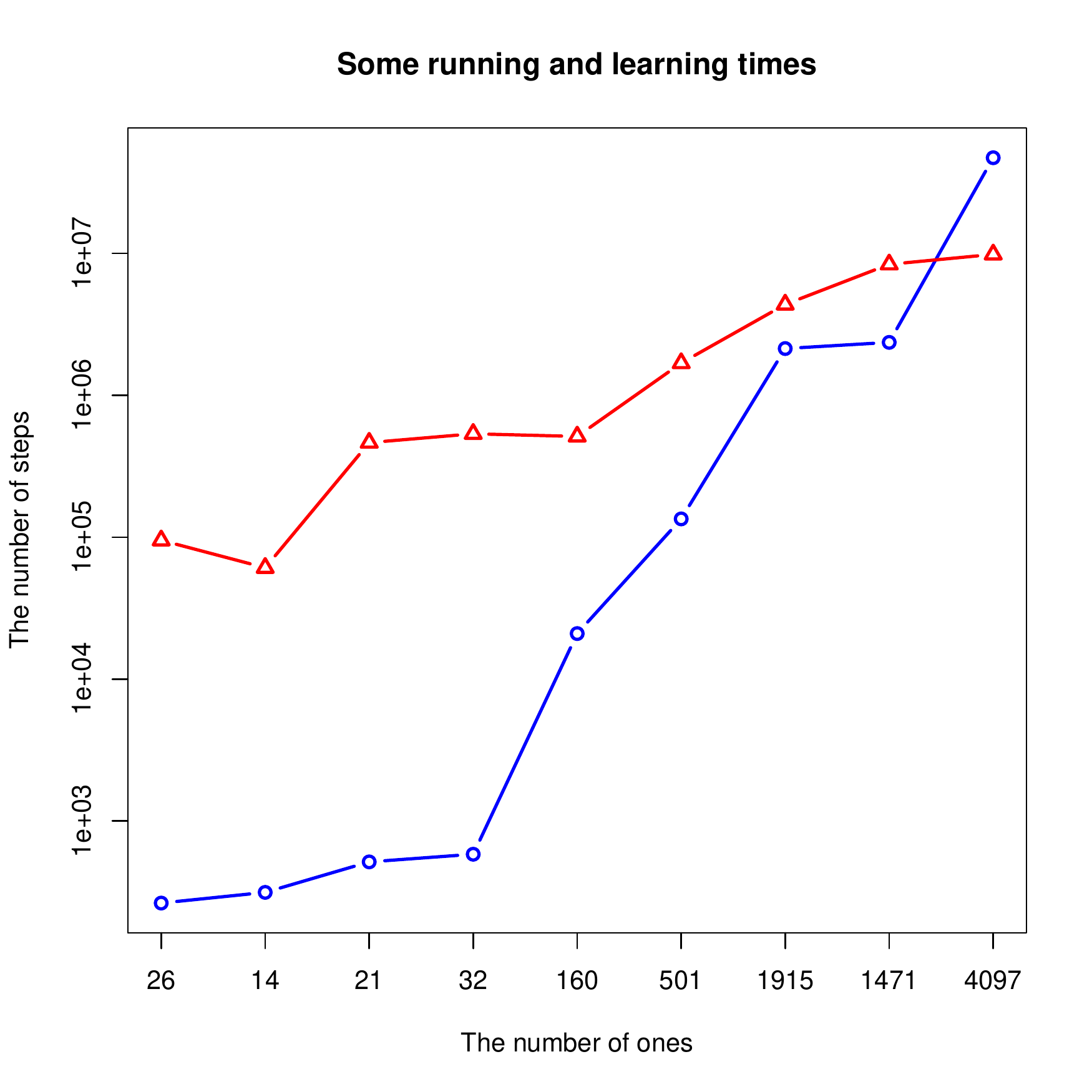}
\caption{This figure shows the usual running time (time complexity) of some given machines and the learning time of these investigated machines. The blue curve is the usual time complexities and the red one is the running times of the learning. The x-axis labeled with the number of ones printed by the Turing machines \enquote{26}, \enquote{14}, \enquote{21}, \enquote{32}, \enquote{160}, Schult (\enquote{501}), Uhing (\enquote{1915}), Uhing (\enquote{1471}) and Marxen-Buntrock (\enquote{4097}). For more precise details see https://github.com/nbatfai/SamuCTuring/releases/tag/vPaperTheorRobopsy and Table \ref{tablecc}.\label{fig_runninglearning}}
\end{figure}

Fig. \ref{fig_runninglearning} summarizes and compares some running results produced by the project Samu C. Turing.
The numbers of two kinds of running times (usual time complexity and  \enquote{learning complexity}, see the caption of the figure for details) are not directly comparable because they use different scales to compute the y-axis values. One of the two curves is computed by the number of steps of a Turing machine and the other by the number of sensory-action pairs of the reinforcement learning agent Samu C. Turing. The exact values can be found in Table \ref{tablecc}.
One of the notions of cognitive complexity defined in Subsect. \ref{cc} will be based on this intuitive \enquote{learning complexity}. In Fig. \ref{fig_runninglearning}, it seems that the growth rate of the learning time is related to the running time. 
It is worth to compare this with Fig. \ref{fig_selfreprocomp} where the growth rate of an another (the \enquote{self-reproducing}) complexity has already been separated from the running time. 

\subsection{The Basic Notions of the Subject}

From the observations of the two experiments above, we can build the abstract model of learning that is referred to as the learning problem. The learning problem of learning TMs is divided into two parts. The first is a simulation of the TM to be learned. The second is the actual learning problem itself. Fig. \ref{fig_learning} shows the schematic of the learning problem where the UTM $R$ takes the description of the machine $T$ and an $x$ input of $T$. Then $R$ has collected the configurations of $T$ whilst it is simulating $T$ with $x$. After the simulation $S$ takes the collected configurations and it must try to figure out what TM was actually simulating.

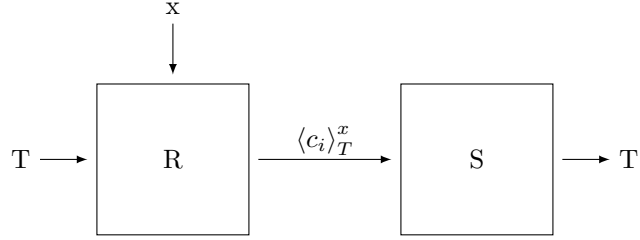
\begin{figure}[!h]
\centering
\scalebox{1}{\begin{tikzpicture}

\draw  (-4,3.5) rectangle (-2,1.5);
\draw  (0,3.5) rectangle (2,1.5);
\node (v2) at (-5,2.5) {T};
\node (v1) at (-4,2.5) {};
\node (v4) at (-2,2.5) {};
\node (v3) at (0,2.5) {};
\node (v6) at (2,2.5) {};
\node (v5) at (3,2.5) {T};
\draw [-latex] (v2) edge (v1);
\draw [-latex] (v4) edge (v3);
\draw [-latex] (v6) edge (v5);
\node (v7) at (-3,3.5) {};
\node (v8) at (-3,4.5) {x};
\draw [-latex] (v8) edge (v7);
\node at (-3,2.5) {R};
\node at (1,2.5) {S};
\node at (-1,2.75) {${\langle c_i\rangle}_T^x$};
\node at (3,2.5) {};
\end{tikzpicture}}
\caption{This figure shows the schematic of the learning problem. The universal machine $R$ takes two input parameters the description $T$ of a TM and the input $x$ of the machine $T$. The machine $R$ computes the sequence ${\langle c_i\rangle}_T^x$ of configurations occurred during the execution of the machine $T$ with its input $x$. Then the learning machine $S$ takes this sequence and finally $S$ has to figure out from this input sequence what was actually simulated by the machine $R$.\label{fig_learning}}
\end{figure}

\subsubsection{The Running Problem\label{rp}}

It is obvious that the running problem trivially contains the halting problem. Therefore we may notice that similar undecidable statements can be made for this case as well but in this paper we only focus on halting machines.  

\begin{lemma}\label{lemma}
Apart from the trivial case of the empty tapes, the transition rule between two consecutive configurations $c_i$ and $c_{i+1}$ is 
%clearly defined 
uniquely determined 
by the configurations $c_i$ and $c_{i+1}$.
\end{lemma}
\begin{proof}
Suppose that there are two transition rules 
$(q, r) \to (q_1, w_1, d_1)$ and $(q, r) \to (q_2, w_2, d_2)$ where 
$q, q_1, q_2 \in Q$, 
$r, w_1, w_2 \in \{0, 1, \text{\textvisiblespace}\}$, 
$d_1, d_2 \in \{\leftarrow,\uparrow,\rightarrow\}$
and then we show that $q_1 = q_2$, $w_1 = w_2$ and $d_1 = d_2$.

Let $c_i = Ll[q>rR$ where $l \in \{0, 1, \text{\textvisiblespace}\}$, $L, R \in \{0, 1, \text{\textvisiblespace}\}^\infty$, 

Then the following cases are possible

\begin{equation*}
  c_{i+1} = Ll[q_1>w_1R, (d_1 = \uparrow)\begin{cases*}
    &  $ = Ll[q_2>w_2R$, $(d_2 = \uparrow)$ $\Leftrightarrow$ ($q_1 = q_2$, \\
    & $\quad$ $w_1 = w_2$) %or ($q_1 = q_2$,
    \\
%    & $\quad$ $w_1 = w_2$ and $d_1 = d_2$)\\    
%
    &  $ = L[q_2>lw_2R$, $(\gets)$ $\Leftrightarrow$ ($q_1 = q_2$,\\ 
    & $\quad$ %$w_1 = w_2$, 
    $Ll = L$, $w_1R = lw_2R$ that is, iff \\
    & $\quad$ $l, w_1, w_2 = \text{\textvisiblespace}$ and $R, L = \text{\textvisiblespace}^\infty$)\\
    &  $ = Llw_2[q_2>R$, $(\to)$ $\Leftrightarrow$ ($q_1 = q_2$,\\ 
    & $\quad$ %$w_1 = w_2$,
     $Ll = Llw_2$, $w_1R = R$ that is, iff \\
    & $\quad$ $l, w_1, w_2 = \text{\textvisiblespace}$ and $R, L = \text{\textvisiblespace}^\infty$)
  \end{cases*}
\end{equation*}

\begin{equation*}
  c_{i+1} = L[q_1>lw_1R, (\gets)\begin{cases*}
    &  $ = L[q_2>lw_2R$, $(\gets)$ $\Leftrightarrow$ ($q_1 = q_2$,\\ 
    & $\quad$ $w_1 = w_2$)\\
    &  $ = Llw_2[q_2>R$, $(\to)$$\Leftrightarrow$ ($q_1 = q_2$,\\ 
    & $\quad$ %$w_1 = w_2$,
     $L = Llw_2$, $lw_1R = R$ that is, iff \\
    & $\quad$ $l, w_1, w_2 = \text{\textvisiblespace}$ and $R, L = \text{\textvisiblespace}^\infty$)
  \end{cases*}
\end{equation*}

\begin{equation*}
  c_{i+1} = Llw_1[q_1>R, (\to)\begin{cases*}
    &  $ = Llw_2[q_2>R$, $(\to)$ $\Leftrightarrow$ ($q_1 = q_2$,\\ 
    & $\quad$ $w_1 = w_2$)
  \end{cases*}
\end{equation*}

\end{proof}

\begin{remark}\label{rem_configinfty}
It is noted that we may give an even more simpler lemma and proof using the usual 
$\text{\textvisiblespace}^*$ and $ \{0, 1, \text{\textvisiblespace}\}^*$
instead of 
$\text{\textvisiblespace}^\infty$ and $ \{0, 1, \text{\textvisiblespace}\}^\infty$.
We use the latter because they are closer to the programmers' intuition.
\end{remark}

\begin{theorem}[Universal Learning]\label{theorem_learning}
There exist an universal running machine $R$ and a learning machine $S$ such that, for all halting Turing machines $T$, it holds that $S(R(T, x)) = T$.
\end{theorem}
\begin{proof}

The proof is divided into two parts: in the first one, we modify the usual proof of Turing's universal simulation theorem (see for example the textbook \cite{RonyaiKonyv}) to produce the sequence of configurations of $T$ by the universal machine $R$. In the other part we focus the learning of $S$ by using the previous lemma.

We provide only an outline of the first part. We use a multitape TM for the implementation of $R$. Fig. \ref{fig_proof1} shows the preparation of the tapes before starting the simulation of $T$. The tapes are shown in Fig. \ref{fig_proof2} after the simulation of the i-th step of $T$.

\begin{figure}[!h]
\centering
\begin{tikzpicture}

\node (v6) at (-4,8.5) {};
\node (v5) at (3,8.5) {};
\draw  (v5) edge (v6);

\node (v8) at (-4,8) {};
\node (v7) at (3,8) {};
\draw  (v7) edge (v8);

\node at (-.5,8.25) {encoded T and x};

\node at (-.5,7.85) {$\vdots$};

\node (v2) at (-4,7.5) {};
\node (v1) at (3,7.5) {};
\draw  (v1) edge (v2);

\node (v4) at (-4,7) {};
\node (v3) at (3,7) {};
\draw  (v3) edge (v4);

\node at (-.5,7.25) {$\triangleright|q_0>x\triangleleft$};

\node (v10) at (-4,6.75) {};
\node (v9) at (3,6.75) {};
\draw  (v9) edge (v10);

\node (v12) at (-4,6.25) {};
\node (v11) at (3,6.25) {};
\draw  (v11) edge (v12);

\node at (-.5,6.5) {};

\end{tikzpicture}
\caption{This figure shows the preparation of the tapes of $R$. On the second last tape $R$ denotes the used cells with the symbols $\triangleright$ and $\triangleleft$. From the point of view of $T$ these symbols are interpreted as the blank symbol \textvisiblespace \ on the tape. But from the point of view of $R$ they may be \enquote{interpreted} as $\text{\textvisiblespace}^\infty$ from left and from right. \label{fig_proof1}}
\end{figure}
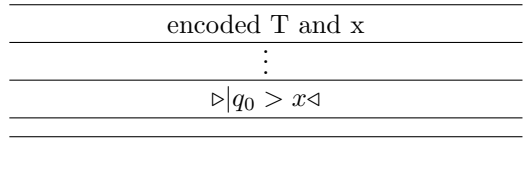

\begin{figure}[!h]
\centering
\begin{tikzpicture}

\node (v6) at (-4,8.5) {};
\node (v5) at (3,8.5) {};
\draw  (v5) edge (v6);

\node (v8) at (-4,8) {};
\node (v7) at (3,8) {};
\draw  (v7) edge (v8);

\node at (-.5,8.25) {encoded T and x};

\node at (-.5,7.85) {$\vdots$};

\node (v2) at (-4,7.5) {};
\node (v1) at (3,7.5) {};
\draw  (v1) edge (v2);

\node (v4) at (-4,7) {};
\node (v3) at (3,7) {};
\draw  (v3) edge (v4);

\node at (-.5,7.25) {$\triangleright c_i\triangleleft$};

\node (v10) at (-4,6.75) {};
\node (v9) at (3,6.75) {};
\draw  (v9) edge (v10);

\node (v12) at (-4,6.25) {};
\node (v11) at (3,6.25) {};
\draw  (v11) edge (v12);

\node at (-.5,6.5) {
$\text{\textvisiblespace}
\triangleright 
c_1\triangleleft\text{\textvisiblespace}
\triangleright 
c_2\triangleleft\text{\textvisiblespace}
\dots{\text{\textvisiblespace}}
\triangleright 
c_i\triangleleft\text{\textvisiblespace}
$};

\end{tikzpicture}
\caption{This figure shows that the denoted configuration $\triangleright c_i\triangleleft$ is copied (and collected) to the output tape after the simulation of the i-th step of $T$.\label{fig_proof2}}
\end{figure}
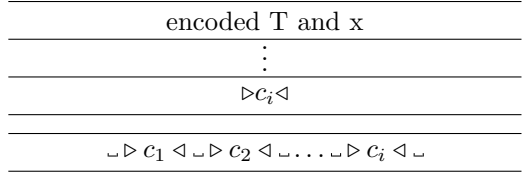

Then the theorem follows from Lemma \ref{lemma}.
\end{proof}

\subsubsection{The Learning Problem\label{lp}}

The previous theorem shows that there is no problem with learning if we use config$\infty$ (or the usual) configurations. But otherwise, as shown in the following two simple examples of config2 configurations (Example \ref{ex1} and \ref{ex2}) the applied transition rule between two consecutive configN configurations may be not uniquely determined by the configN configurations. If we use configN configurations instead of the usual or config$\infty$ configurations then the Lemma \ref{lemma} does not hold. In the next subsection a notion of complexity will be exactly based on this property.

\begin{example}\label{ex1}
Let $c_i =\text{\textvisiblespace}^\infty11111[q>11111\text{\textvisiblespace}^\infty$ 
be a config$\infty$ configuration and $c_i^, $ be a corresponded config2 configuration. Then the rules 
$(q, 1) \to (q, 1, \gets)$,
$(q, 1) \to (q, 1, \to)$,
and $(q, 1) \to (q, 1, \uparrow)$ yield the same $c_{i+1}^, = 11[q>11$ config2 configuration. 
\end{example}

\begin{example}\label{ex2}
Let $c_i =\text{\textvisiblespace}^\infty0101[q>1101\text{\textvisiblespace}^\infty$ 
be a config$\infty$ configuration and $c_i^, $ be a corresponded config2 configuration. Then the rules 
$(q, 1) \to (q, 0, \gets)$ and
$(q, 1) \to (q, 0, \to)$ yield the same $c_{i+1}^, = 10[q>10$ config2 configuration. 
\end{example}
 
\subsubsection{Cognitive Complexities\label{cc}}

As has already been mentioned in Sect. \ref{intuitive} we intuitively use the running time of the learning machines as a complexity measure that may be formulated as follows
\[cc(T,x) = min\{t_S\left({\langle c_i\rangle}_T^x\right) \vert T(x)<\infty, S(R(T, x)) = T\}\]
but it does not seem very helpful because it is probably correlated with the usual time complexity of $T$ as it is suggested by Fig. \ref{fig_runninglearning}. The next type of complexity tells what is the first finite $N$ for which Lemma \ref{lemma} holds with using the configurations $configN$. To be more precise, it is defined as 
\[cc*(T,x) = min\{N \vert T(x)<\infty, S(R(T, x)) = T \text{\ and for configN the lemma \ref{lemma} holds}\}\]
that has shown different behavior than the previous one as it can be seen in Fig. \ref{fig_selfreprocomp} 
The growth rate of the investigated $cc*$ values not related to the number of ones rather than to the running time (see \enquote{14}, \enquote{21} and \enquote{1471}).

\begin{figure}[!h]
\centering
\includegraphics[scale=.5]{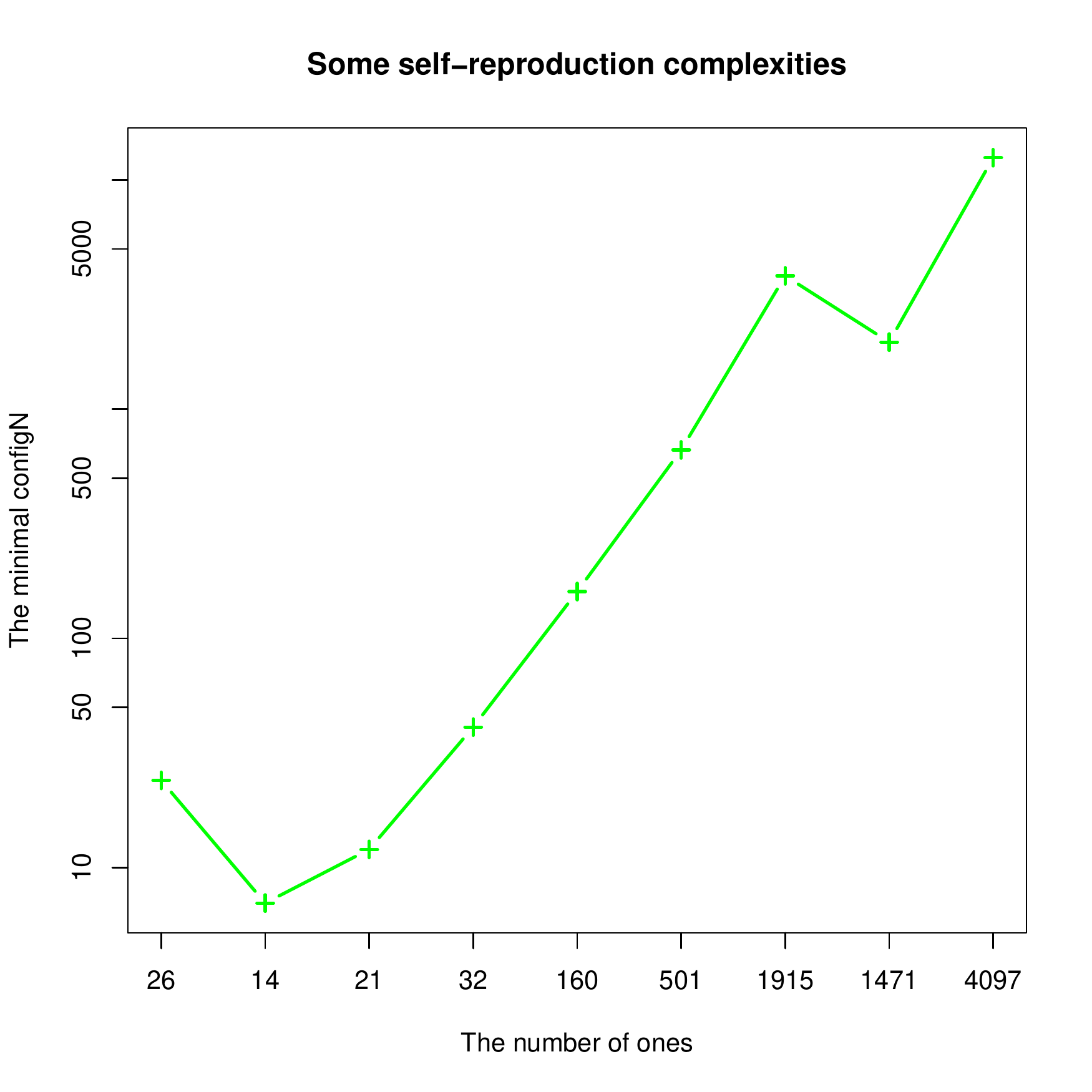}
\caption{This figure shows the $cc*$ values of machines of Fig. \ref{fig_runninglearning}. The values are computed by the version of the project Samu C. Turing that tagged by \texttt{self-reproducing\_complexity}, see  
\texttt{https://github.com/nbatfai/SamuCTuring/tree/self-reproducing\_complexity} where a manual binary search was also used to determine the last three $cc*$ values. The x-axis is exactly the same as in Fig. \ref{fig_runninglearning}.\label{fig_selfreprocomp}}
\end{figure}

The results shown in Fig. \ref{fig_selfreprocomp} also suggest that it is hopeless to handle the learning problem with the universal learning machine $S$ of Lemma \ref{lemma}. This justifies the using of agent technology (an agent observes the operation of the investigated TMs) and machine learning algorithms (such as Q-learning) to learn Turing machines instead of searching for suitable configNs for any universal learning machine $S$.

\begin{table}[!h]
\begin{center}
 \begin{tabularx}{.9\textwidth}{|r|r|X|r|} 
 \hline
 $t_T$ & 1s of $T(\lambda)$ & $cc(T, \lambda)$ & $\operatorname{cc\ast}(T, \lambda)$ \\  
 \hline
\rowcolor{Gray}
 \multicolumn{4}{|l|}{9, 0, 9, 1, 11, 2, 17, 3, 21, 4, 19, 5, 29, 6, 5, 7, 6, 8, 8}\\
  \hline
264    &  26 &    95048 &  24\\
 \hline
 \rowcolor{Gray}
 \multicolumn{4}{|l|}{9, 0, 9, 1, 11, 2, 5, 3, 20, 4, 17, 5, 24, 7, 29, 8, 15, 9, 1}\\
  \hline
314     & 14 &    60872 &  7\\
 \hline
 \rowcolor{Gray}
 \multicolumn{4}{|l|}{9, 0, 9, 1, 11, 2, 15, 3, 20, 4, 21, 5, 27, 6, 4, 7, 2, 8, 12}\\
  \hline
515     & 21 &    463558 & 12\\
 \hline
 \rowcolor{Gray}
 \multicolumn{4}{|l|}{9, 0, 21, 1, 9, 2, 24, 3, 6, 4, 3, 5, 20, 6, 17, 7, 0, 9, 15}\\
  \hline
583    &  32   &  535050 & 41\\
 \hline
 \rowcolor{Gray}
 \multicolumn{4}{|l|}{9, 0, 9, 1, 12, 2, 15, 3, 21, 4, 29, 5, 1, 7, 24, 8, 2, 9, 27}\\
  \hline
20928 &   160   & 512623 & 160\\
 \hline
 \rowcolor{Gray}
 \multicolumn{4}{|l|}{\parbox[t][][t]{.7\textwidth}{9, 0, 11, 1, 12, 2, 17, 3, 23, 4, 3, 5, 8, 6, 26, 8, 15, 9, 5 (Schult's machine)}}\\
  \hline
134467 &  501   & 1685939 &664\\
 \hline
 \rowcolor{Gray}
 \multicolumn{4}{|l|}{\parbox[t][][t]{.7\textwidth}{9, 0, 11, 1, 15, 2, 0, 3, 18, 4, 3, 6, 9, 7, 29, 8, 20, 9, 8  (Uhing's machine)}}\\
  \hline
2133492 & 1915 &  4365184 &3816\\
 \hline
 \rowcolor{Gray}
 \multicolumn{4}{|l|}{\parbox[t][][t]{.7\textwidth}{9, 0, 11, 2, 15, 3, 17, 4, 26, 5, 18, 6, 15, 7, 6, 8, 23, 9, 5 (Uhing's machine)}}\\
  \hline
2358064 & 1471 &  8368208 &1961\\ 
 \hline
 \rowcolor{Gray}
 \multicolumn{4}{|l|}{\parbox[t][][t]{.7\textwidth}{9, 0, 11, 1, 15, 2, 17, 3, 11, 4, 23, 5, 24, 6, 3, 7, 21, 9, 0 \\(Marxen and Buntrock's BB5 champion machine)}}\\
  \hline
47176870& 4097  & 9833455 &12287\\ 
 \hline
\end{tabularx}
\end{center}
\caption{This table numerically shows the $cc*$ values of the investigated machines. The combine columns show the given TM in the form of rule-index notation \cite{batfai-orchmach}.\label{tablecc}}
\end{table}

\section{Conclusion}

In this paper, we started with two developmental robotic software experiments Samu Turing \cite{SamuTuring} and Samu C. Turing \cite{SamuCTuring} to learn how Turing machines operate. This subject of the experiments itself enabled us to investigate the theoretical properties of learning. First, we have eliminated from our software experiments the developmental robotic processes (for example the habituation-sensitization parts) and then we introduced the problem of learning and some complexity measures based on it. For some cases of given TMs we also determine these complexities. 
The $\operatorname{cc\ast}$ of machines of greater sophistication cannot easily be computed by the universal learning machine $S$ of Theorem \ref{theorem_learning}. This justifies the usage of agent technology and machine learning for learning Turing machines. We have provided only an outline of the proof of Theorem \ref{theorem_learning}. To complete it may be a further theoretical computer science work. Further work of a practical robopsychological nature is also needed. For example, we are going to investigate using Samu's neural architecture \cite{samu}, Samu mental organs (like MPUs) \cite{psamu1} and deep learning to learn how TMs operate. 

To return to Neumann's train of thought mentioned in the introduction it seems to be interesting to study when the learning algorithm has been applied to write itself. Let's start from a machine $T$ that halts with $x$. It follows from Theorem \ref{theorem_learning} that 
$R(T, x) = {\langle c_i\rangle}_T^x$ and $S(R(T, x)) = T$. 
But then we can also learn this learning of $T$, that is 
$R(S, {\langle c_i\rangle}_T^x) = {\langle c_i\rangle}_S^{{\langle c_i\rangle}_T^x}$ and $S(R(S, {\langle c_i\rangle}_T^x)) = S$. And then we can learn again the learning of learning of $T$, that is, to be more precise $R(S, {\langle c_i\rangle}_S^{{\langle c_i\rangle}_T^x}) = {\langle c_i\rangle}_S^{{\langle c_i\rangle}_S^{{\langle c_i\rangle}_T^x}}$ and so on.  If we introduce the notation
\[
y^j = {\langle c_i\rangle}_S^{\iddots^{{\langle c_i\rangle}_S^{{\langle c_i\rangle}_T^x}}}
\]
then we can easily write that $cc(S, y^j) < cc(S, y^{j+1})$ because $t_S(y^j) < t_S(y^{j+1})$ but the similar relation between $\operatorname{cc\ast}(S, y^j)$ and $\operatorname{cc\ast}(S, y^{j+1})$ is an open question at this moment.

It is clear, of course, that further work of a theoretical robopsychological nature is required as well. For example, we are going to find possible relations among the time, space, Kolmogorov and cognitive complexities. We believe that this is a necessary step towards achieving the situation that has been defined as \enquote{Programs hacking programs} by Neo in the movie \enquote{The Matrix Reloaded}. In the framework of Turing machines and Busy Beaver problem this quotation has a special meaning namely that can we program a computer program not only to discover a BB machine but to build it from scratch? 

\section{Acknowledgment}
The author would like to thank his students in \enquote{High Level Programming Languages} course in the spring semester of 2015/2016 at the University of Debrecen for testing the Samu projects.
He would also like to thank the members of some AI-specific communities on Facebook, Google+ and Linkedin and especially his group called DevRob2Psy at \url{https://www.facebook.com/groups/devrob2psy/} for their interest.

\bibliography{TheorRobopsy}

\end{document}